\documentclass{article}

\usepackage[preprint]{neurips_2024}

\usepackage[utf8]{inputenc} 
\usepackage[T1]{fontenc}    
\usepackage{hyperref}       
\usepackage{url}            
\usepackage{booktabs}       
\usepackage{amsfonts}       
\usepackage{nicefrac}       
\usepackage{microtype}      
\usepackage{xcolor}         
\usepackage{amsmath}
\usepackage{amsthm}
\usepackage{graphicx}
\usepackage{color}
\newtheorem{theorem}{Theorem}[section]
\newtheorem{definition}{Definition}[section]
\newtheorem{lemma}{Lemma}[section]

\title{Measuring Over-smoothing beyond Dirichlet energy}

\author{
  Weiqi Guan\thanks{School of Mathematical Sciences, Fudan University, Shanghai 200433, China} \\
  \texttt{24110180016@m.fudan.edu.cn} \\
  \And
  Zihao Shi\thanks{School of Mathematical Sciences, Fudan University, Shanghai 200433, China} \\
  \texttt{23210180116@m.fudan.edu.cn} \\
}

\begin{document}

\maketitle

\begin{abstract}
  While Dirichlet energy serves as a prevalent metric for quantifying over-smoothing, it is inherently restricted to capturing first-order feature derivatives. To address this limitation, we propose a generalized family of node similarity measures based on the energy of higher-order feature derivatives. Through a rigorous theoretical analysis of the relationships among these measures, we establish the decay rates of Dirichlet energy under both continuous heat diffusion and discrete aggregation operators. Furthermore, our analysis reveals an intrinsic connection between the over-smoothing decay rate and the spectral gap of the graph Laplacian. Finally, empirical results demonstrate that attention-based Graph Neural Networks (GNNs) suffer from over-smoothing when evaluated under these proposed metrics.
\end{abstract}

\section{Introduction}
Graph Neural Networks (GNNs) have emerged as a dominant paradigm for graph representation learning, demonstrating remarkable success across diverse domains ranging from protein structure prediction~\citep{gligorijevic2021structure} to social recommendation systems~\citep{9139346}. While foundational architectures following the message-passing paradigm—such as Graph Convolutional Networks (GCN)~\citep{kipf2017semisupervised} and Graph Attention Networks (GAT)~\citep{veličković2018graph}—achieve state-of-the-art performance, they frequently encounter the critical issue of \textit{over-smoothing}, where node representations become indistinguishable as the network depth increases.

A mainstream research direction views GNNs through the lens of dynamical systems and differential equations. This view make it easy to understand the dynamical behavior of GNNs and hence finding solutions to overcome over-smoothing. For instance, SGC~\citep{pmlr-v97-wu19e} decouples feature transformation from propagation, demonstrating that simplified propagation schemes can match GCN performance, controling the time step to build deep SGC. This idea has been extended using heat kernels~\citep{wang2021dissecting} and generalized propagation kernels~\citep{pmlr-v162-li22h}. More recently, Ordinary Differential Equation (ODE) and Partial Differential Equation (PDE) inspired architectures have been introduced to regulate the system dynamics. Notable examples include gradient gating mechanisms~\citep{rusch2023gradient}, fractional time derivatives~\citep{kang2024unleashing}. For a comprehensive review of equation-based GNNs, we refer to~\citep{pmlr-v162-rusch22a, NEURIPS2023_2a514213, behmanesh2023tide, thorpe2022grand, chamberlain2021grand, pmlr-v119-xhonneux20a,wang2023acmp,pmlr-v202-choi23a}. 

Extensive studies have theoretically analyzed this degradation in expressive power within both attention mechanisms~\citep{NEURIPS2023_6e4cdfdd} and graph convolutional architectures~\citep{Oono2020Graph, cai2020note}. From a geometric perspective, \cite{10.5555/3618408.3619488} attribute over-smoothing to edges with highly positive Ollivier–Ricci curvature, proposing curvature-based rewiring. Concurrently, spectral analyses link the phenomenon to the spectral properties of the graph Laplacian, specifically the first nonzero eigenvalue~\citep{jamadandi2024spectral, 10.1145/3583780.3614997}, hence cause a trade-off between over-smoothing and over-squashing.

Quantifying over-smoothing remains a fundamental challenge. Several metrics have been proposed: Dirichlet energy, Total pairwise square distance~\citep{Zhao2020PairNorm:}, Group Distance Ratio and Instance Information Gain~\citep{10.5555/3495724.3496137}, Mean-Average Distance (MAD)~\citep{Chen2019MeasuringAR}, and more recently, similarity measures relative to the mean feature~\citep{NEURIPS2023_6e4cdfdd}. Among these, \textit{Dirichlet energy} is the most intuitively appealing and widely used metric~\citep{pmlr-v162-rusch22a, NEURIPS2021_b6417f11}. However, standard Dirichlet energy only captures information regarding the first-order derivative of features. A recent survey~\citep{rusch2023survey} explicitly raises the open question: \textit{Are there other effective node similarity measures?}.

Motivated by these observations, we aim to answer this question. Our main contributions are summarized as follows:
We propose a generalized family of node similarity measures to characterize over-smoothing at a finer granularity, encompassing Dirichlet energy as a special case. We prove that these measures satisfy the rigorous conditions outlined in~\citep{rusch2023survey}. Furthermore, we theoretically derive the equivalence and spectral difference between these measures and establish the decay rate of Dirichlet energy under both continuous heat diffusion and general discrete reversible random walks. Additionally, our continuous diffusion analysis elucidates the intrinsic connection between over-smoothing and the first nonzero eigenvalue of the graph Laplacian.
 
\section{Preliminaries and Notation}

In this section, we introduce the framework of calculus on weighted graphs, which is widely adopted for studying differential equations on graphs~\citep{GRIGORYAN20164924}. Utilizing this framework, we demonstrate that attention-based GNNs can be naturally generalized within this context. 

Throughout this paper, all graphs are assumed to be finite and undirected. Let $G=(V, E, \omega, \mu)$ be a weighted graph, where $V=\{1, \dots, n\}$ is the set of nodes, and $E \subset V \times V$ is the set of edges. The edge weight function $\omega: E \rightarrow \mathbb{R}_{>0}$ satisfies symmetry, i.e., $\omega_{ij} = \omega([i,j]) = \omega([j,i]) > 0$ for all $[i,j] \in E$. The node measure $\mu: V \rightarrow \mathbb{R}_{>0}$ assigns a positive weight $\mu_i = \mu(i) > 0$ to each node $i \in V$. For a fixed node $i \in V$, its neighborhood is denoted by $\mathcal{N}_i = \{j \in V : \omega_{ij} > 0, j \neq i\}$. 

The total measure of the graph is given by $\vert V\vert_{\mu} = \sum_{i \in V} \mu_{i}$. In the sequel, we define the generalized gradient and Laplacian for vector-valued functions on $G$.

\begin{definition}
    Let $G=(V, E, \omega, \mu)$ be a weighted graph and let $f, g: V \rightarrow \mathbb{R}^d$ be vector-valued functions on the nodes.
    \begin{itemize}
        \item The \textbf{integration} of $f$ on $V$ with respect to $\mu$ is defined as:
        \[
            \int_{G} f d\mu = \sum_{i=1}^n f(i) \mu_i.
        \]
        \item The \textbf{inner product} of the gradients of $f$ and $g$ is defined as:
        \begin{align*}
            \int_{G} \nabla_{\mu} f \cdot \nabla_{\mu} g d\mu &= \sum_{i \in V} \sum_{j \in \mathcal{N}_{i}} \frac{\omega_{ij} (f(j)-f(i)) \cdot (g(j)-g(i))}{2\mu_{i}}\mu_i\\
            &= \frac{1}{2} \sum_{i \in V} \sum_{j \in \mathcal{N}_{i}} \omega_{ij} (f(j)-f(i)) \cdot (g(j)-g(i)),
        \end{align*}
        where $\cdot$ denotes the standard dot product in $\mathbb{R}^d$.
        \item The \textbf{$p$-norm of the gradient} of $f$ at node $i$ (for $p \geq 1$) is defined as:
        \[
        \Vert \nabla_{\mu} f \Vert_p(i) = \left( \sum_{j \in \mathcal{N}_i} \frac{\omega_{ij} \Vert f(j)-f(i) \Vert_p^p}{2\mu_i} \right)^{\frac{1}{p}},
        \]
        where $\Vert x \Vert_p$ denotes the vector $p$-norm.
        \item The \textbf{$\mu$-Laplacian} of $f$ at node $i$ is defined as:
        \[
            \Delta_{\mu} f(i) = \sum_{j \in \mathcal{N}_{i}} \frac{\omega_{ij} (f(j)-f(i))}{\mu_{i}}.
        \]
    \end{itemize}
\end{definition}

A fundamental property of the above definitions is the validity of the integration by parts formula, which allows us to perform calculus on graphs. We state the result below; the proof, along with other proofs in this paper, is deferred to the Appendix.

\begin{theorem}[\textbf{Integration by Parts}]
    Let $G=(V, E, \omega, \mu)$ be a weighted graph and let $f, g: V \rightarrow \mathbb{R}^d$ be vector-valued functions. Then:
    \[
        \int_{G} \Delta_{\mu} f \cdot g d\mu = -\int_{G} \nabla_{\mu} f \cdot \nabla_{\mu} g d\mu = \int_{G} f \cdot \Delta_{\mu} g d\mu.
    \]
\end{theorem}

We now relate this general framework to traditional graph Laplacians. Let $A \in \mathbb{R}^{n \times n}$ be the adjacency matrix, and $D = \text{diag}(D_1, \dots, D_n)$ be the degree matrix. Define $\tilde{A} = A + I$ and $\tilde{D}$ as the corresponding matrices with self-loops. The symmetric normalized and row-normalized adjacency matrices are given by $\tilde{A}_{\text{sym}} = \tilde{D}^{-\frac{1}{2}}\tilde{A}\tilde{D}^{-\frac{1}{2}}$ and $\tilde{A}_{\text{rw}} = \tilde{D}^{-1}\tilde{A}$, respectively (similarly for $A_{\text{sym}}$ and $A_{\text{rw}}$ when no isolated nodes exist). The standard Laplacians are defined as $\Delta_{\text{adj}} = A - D$, $\tilde{\Delta}_{\text{adj}} = \tilde{A} - \tilde{D}$, $\Delta_{\text{sym}} = A_{\text{sym}} - I$, and $\Delta_{\text{rw}} = A_{\text{rw}} - I$.
These can be recovered as special cases of the general $\mu$-Laplacian by specifying weights:
\[
\left\{
\begin{array}{ll}
    \Delta = \Delta_{\text{adj}} & \text{if } \omega_{ij}=1, \mu_i=1, \\
    \Delta = \tilde{\Delta}_{\text{adj}} & \text{if } \omega_{ij}=1, \mu_i=1 \text{ (with self-loops)}, \\
    \Delta = \tilde{\Delta}_{\text{rw}} & \text{if } \omega_{ij}=1, \mu_i=D_i+1, \\
    \Delta = \Delta_{\text{rw}} & \text{if } \omega_{ij}=1, \mu_i=D_i.
\end{array}
\right.
\]

According to Theorem 2.1, $-\Delta_{\mu}$ is a positive semi-definite symmetric operator on the function space $L^2(V, d\mu) = \{f: V \rightarrow \mathbb{R}^d : \int_{G} \Vert f \Vert_2^2 d\mu < \infty\}$. We denote its eigenvalues (without multiplicity) as:
\[
0 = \lambda_0(-\Delta_{\mu}) < \lambda_1(-\Delta_{\mu}) < \dots < \lambda_N(-\Delta_{\mu}),
\]
where $N \geq 1$. Let $M_{\max} = \max_{i \in V} \frac{\sum_{j \in \mathcal{N}_i} \omega_{ij}}{\mu_i}$. We have the following spectral bound:

\begin{theorem}
    For any weighted graph $G=(V, E, \omega, \mu)$, the largest eigenvalue satisfies:
    \[
    \lambda_N(-\Delta_{\mu}) \leq 2M_{\max}.
    \]
\end{theorem}
Specifically, if $\sum_{j \in \mathcal{N}_i} \omega_{ij} \leq \mu_i$ for all $i$, then $\lambda_N(-\Delta_{\mu}) \leq 2$, furthermore we set $P_{\mu} = \Delta_{\mu} + I$ as the random walk matrix.

Finally, recall a typical layer update for attention-based GNNs is given by~\citep{NEURIPS2023_6e4cdfdd}:
\[
X^{l+1} = \phi(P^{l}X^{l}W^{l}),
\]
where $\phi$ is an activation function, $W^{l}$ is a learnable weight matrix, and $P^{l}$ is an aggregation operator. The entries $P_{ij}^{l}$ are commonly formulated as softmax-normalized attention coefficients:
\[
P_{ij}^l = \frac{\exp(e^l_{ij})}{\exp(e_{ii}^l) + \sum_{k \in \mathcal{N}_i} \exp(e^l_{ik})},
\]
where $e^l_{ij}$ is the attention score between node $i$ and $j$. Compared to analysis in \citep{NEURIPS2023_6e4cdfdd}, we only further assume $P^l$ is reversible, that is there exist function $\mu:V\rightarrow \mathbb{R}_{>0}$ such that 
\[
P^{l}(i,j)\mu(i)=P^{l}(j,i)\mu(j)
\]
Then we can set the node weight as $\mu$ and $\omega_{ij}=P^{l}(i,j)\mu(i)$. Or we assume $e_{ij} = e_{ji}$, this is not strict since GCN satisfies this condition. then we set $\omega_{ij}=exp(e_{ij}^l)$ and $\mu_{i}=exp(e_{ii}^l)+\sum\limits_{j\in\mathcal{N}_i}exp(e_{ij}^l).$
  Assuming reversible will bring us much beautiful theory results and easy understanding of over-smoothing.
\section{Theoretical analysis of proposed measures}
\subsection{Measures of Over-smoothing}

\citet{NEURIPS2023_6e4cdfdd} introduced a node similarity measure defined as:
\[
    \mathcal{E}_{W}(X)=\frac{1}{n}\Vert X-\mathbf{1}_{\gamma_X}\Vert_2^2,
\]
where $\mathbf{1}_{\gamma_X} = \frac{1}{n}\sum_{i \in [n]} X_i \cdot \mathbf{1}$ denotes the mean feature vector. Another widely adopted metric is the Dirichlet energy, given by:
\[
     \mathcal{E}_{D}(X) = \frac{1}{n}\sum_{i\in[n]}\sum_{j\in\mathcal{N}_i}\Vert X_i-X_j\Vert_2^2.
\]
Generally, \citet{rusch2023survey} formalized the concepts of node similarity measures and over-smoothing as follows:

\begin{definition}[\textbf{Over-smoothing}]
    Let $G$ be an undirected, connected graph, and let $X^k \in \mathbb{R}^{n \times m}$ denote the node features at the $k$-th layer of an $N$-layer GNN. A functional $\mathcal{E}: \mathbb{R}^{n \times m} \rightarrow \mathbb{R}_{\geq 0}$ is called a \textit{node similarity measure} if it satisfies the following axioms:
    \begin{itemize}
        \item $\mathcal{E}(X) = 0$ if and only if there exists a constant vector $c \in \mathbb{R}^{m}$ such that $X_{i} = c$ for all $i \in V$.
        \item $\mathcal{E}(X+Y) \leq \mathcal{E}(X) + \mathcal{E}(Y)$ for all $X, Y \in \mathbb{R}^{n \times m}$ (Triangle Inequality).
    \end{itemize}
    We define \textit{over-smoothing} with respect to $\mathcal{E}$ as the layer-wise exponential convergence of the measure to zero. Specifically, for layers $k=0, \dots, N$, there exist constants $C_1, C_2 > 0$ such that:
    \begin{equation*}
        \mathcal{E}(X^k) \leq C_1 e^{-C_2 k}.
    \end{equation*}
\end{definition} 

We observe that assuming $\mathbf{1}_{\gamma_X}=0$, $\mathcal{E}_{W}(X)$ reduces to $\Vert X\Vert_2^2$, which represents the energy of the zero-order feature derivative. Similarly, the Dirichlet energy corresponds to the energy of the first-order derivative. Motivated by the limitation that existing metrics capture only low-order information, we propose a generalized series of node similarity measures designed to capture higher-order derivatives. Furthermore, we theoretically demonstrate that attention-based GNNs suffer from over-smoothing under these proposed measures.

\begin{theorem}
    Let $G = (V, E, \omega, \mu)$ be a connected weighted graph. For a vector-valued function $f$ on the graph, we define the $m$-th derivative energy $\mathcal{E}_{m}(f)$ as:
    \[
    \mathcal{E}_{m}(f) := \frac{1}{n}\int_{G}\Vert\nabla_{\mu}^m f\Vert_2^{2} d\mu,
    \]
    where $m \in \mathbb{N}$, and the higher-order derivatives are defined as:
    \[
        \Vert\nabla_{\mu}^m f\Vert_2 = \left\{
        \begin{array}{ll}
               \Vert(-\Delta_{\mu})^{\frac{m}{2}}f\Vert_{2} & \text{if } m \text{ is even}, \\
               \Vert\nabla_{\mu}(-\Delta_{\mu})^{\frac{m-1}{2}}f\Vert_{2} & \text{if } m \text{ is odd}.
        \end{array}
        \right.
    \]
    Here, the zero-order operator is defined as $(-\Delta_{\mu})^0 f = \lim_{s\rightarrow 0^+} (-\Delta_{\mu})^{s} f = f - \frac{1}{\vert V\vert_{\mu}}\int_{G}f d\mu$, and the fractional Laplacian follows the definition in \citep{NEURIPS2023_2a514213}. 
    
    The node similarity defined by $\gamma_{m}(f) = (\mathcal{E}_{m}(f))^{\frac{1}{2}}$ satisfies the axioms of a node similarity measure. Moreover, $\mathcal{E}_{W}$ and $\mathcal{E}_{D}$ are special cases of our proposed measures (corresponding to $m=0$ and $m=1$, respectively).
\end{theorem}
Suppose $\Delta_{\mu} = P_{\mu} - I$. Assuming the existence of a steady state $X^{\infty}$ such that $\lim_{k \to \infty} P_{\mu}^{k}X = X^{\infty}$ for any $X$, the energy functional $\mathcal{E}_2(X) = \frac{1}{n} \int_{G} \Vert P_{\mu}X - X \Vert^2_2 d\mu$ quantifies the deviation of $X$ from the steady state $X^{\infty}$. Consequently, $\mathcal{E}_2$ provides a more general metric for assessing convergence to equilibrium compared to the standard Dirichlet energy.
\subsection{Equivalence of Node Similarity Measures}
We now investigate the relationship between these measures. We establish that for $m \geq 1$, the measures $\mathcal{E}_{m}(f)$ are all equivalent. If the graph is connected, this equivalence extends to all $m \geq 0$ (under the zero-mean assumption for $m=0$). The case for $m \leq 1$ follows directly from the Poincaré inequality Assuming the graph is connected, we now establish the equivalence between $\mathcal{E}_0$ and $\mathcal{E}_1$, which is fundamentally grounded in the Poincaré inequality. The existence of the Poincaré inequality is intrinsically linked to key geometric properties of the graph, such as heat kernel estimates and the volume doubling property~\citep{HornLinLiuYau+2019+89+130}. The inequality is formally stated as follows:
\begin{theorem}
    Let $G = (V,E,\omega,\mu)$ be a connected weighted graph. here exist a constaTnts $C > 0$ such that for any vector-valued function $f$ on $G$, the following Poincaré inequality holds:
    \begin{equation*}
        \lambda_1(-\Delta_{\mu})\int_{G}\left\Vert f - \frac{1}{\vert V\vert_{\mu}}\int_{G}f d\mu\right\Vert_2^2 d\mu \leq\int_{G}\Vert\nabla_{\mu} f\Vert_2^2 d\mu.
    \end{equation*}
   Furthermore, we have the converse inequality (boundedness of the gradient):
    \[
    \int_{G}\Vert\nabla_{\mu} f\Vert_2^2 d\mu \leq C\int_{G}\left\Vert f - \frac{1}{\vert V\vert_{\mu}}\int_{G}f d\mu\right\Vert_2^2 d\mu.
    \]
    Consequently, the measures $\mathcal{E}_0$ and $\mathcal{E}_1$ are equivalent norms on the quotient space orthogonal to constant functions.
\end{theorem}. 
Now, we present the result for $m \geq 1$.

\begin{theorem}
    Let $G = (V, E, \omega, \mu)$ be a weighted graph. For any function $f$ defined on the graph, there exist universal constants $C_1, C_2, C_3, C_4 > 0$ such that:
    \[
        C_2 \int_{G} \Vert\nabla_{\mu} f\Vert_2^2 d\mu \leq \int_{G} \Vert\Delta_{\mu} f\Vert_2^2 d\mu \leq C_1 \int_{G} \Vert\nabla_{\mu} f\Vert_2^2 d\mu,
    \]
    and
    \[
        C_3 \int_{G} \Vert\nabla_{\mu} \Delta_{\mu} f\Vert_2^2 d\mu \leq \int_{G} \Vert\Delta_{\mu} f\Vert_2^2 d\mu \leq C_4 \int_{G} \Vert\nabla_{\mu} \Delta_{\mu} f\Vert_2^2 d\mu.
    \]
    Specifically, the constant $C_2$ corresponds to the first non-zero eigenvalue $\lambda_1$.
\end{theorem}
 By recursively replacing $f$ with $\Delta_{\mu}^k g$, the inequalities in Theorem 3.3 can be generalized to establish equivalence for all $m \geq 1$. This implies that if a GNN suffers from over-smoothing with respect to one of these measures, it will inevitably suffer from over-smoothing with respect to the others.

\subsection{Spectral difference of proposed measures}
while the preceding subsection established the equivalence of the proposed node similarity measures, we now distinguish them through a spectral lens. Let $\{\alpha_k\}_{k=1}^{nd}$ denote the eigenvalues of $-\Delta_{\mu}$ sorted in ascending order, i.e., $\alpha_1 \leq \alpha_2 \leq \dots \leq \alpha_{nd}$. Let $\{v_k\}_{k=1}^{nd}$ be the corresponding orthonormal eigenfunctions satisfying $\int_{G} v_i \cdot v_j d\mu = \delta_{ij}$.

For any function $f \in C(G)$, its spectral decomposition with respect to $-\Delta_{\mu}$ is given by $f = \sum_{k=1}^{nd} C_k v_k$. Consequently, the Dirichlet energy can be derived as:
\[
\int_{G}\Vert\nabla_{\mu} f\Vert^2_2 d\mu = \int_{G} (-\Delta_{\mu} f) \cdot f d\mu = \int_{G} \left(\sum_{k=1}^{nd}\alpha_k C_{k}v_k\right) \cdot \left(\sum_{k=1}^{nd} C_k v_k\right) d\mu = \sum_{k=1}^{nd}\alpha_k C_k^2.
\]
By extension, the expression for the general $m$-th derivative energy is formulated as follows
\begin{theorem}\label{Thm:spectrum decomposition}
    Given a weighted graph $G=(V,E,\omega,\mu)$ and a function $f\in C(G)$, we have following formula
    \[
    \int_{G}\Vert\nabla_{\mu}^m f\Vert^2_2d\mu=\sum\limits_{1\leq k\leq nd}\alpha^m_kC_k^2.
    \]
\end{theorem}
According to Theorems 2.1 and 2.2, the eigenvalues satisfy $0\leq \alpha_k\leq 2$ for all $1\leq k\leq nd$. Consequently, as $m$ increases, the spectral contribution of high-frequency components ($\alpha_k > 1$) is amplified, whereas low-frequency components ($\alpha_k < 1$) are attenuated. This property allows us to vary $m$ to capture distinct spectral information of the features.
\begin{section}{Theoretical analysis of over-smoothing}
    In this section, we investigate the over-smoothing phenomenon in Simple Graph Convolution (SGC) with general reversible random walks and establish its relationship with the spectral gap. Over-smoothing in message-passing GNNs is typically attributed to spatial propagation. SGC~\citep{wang2021dissecting} simplifies GCN by removing the non-linear transformation between layers while achieving competitive performance. Its formulation is given by:
\[
Y = \text{softmax}(\tilde{A}^K_{\text{sym}}XW),
\]
where $K$ denotes the number of layers and $W$ represents the learnable parameters. \citet{wang2021dissecting} demonstrated that the feature propagation stage in SGC is equivalent to heat diffusion on the graph with a fixed step size $\Delta t=1$. Since heat diffusion causes features within each connected component to converge to a constant vector, they proposed controlling the time step size to mitigate over-smoothing. In this work, we extend this analysis to general reversible random walks and derive a sharp decay rate characterizing the over-smoothing process.  
\begin{theorem}
    Let $G=(V, E,\omega,\mu)$ be a weighted graph. If $f_t$ is a solution of the heat equation
    \[
         \frac{\partial f_t}{\partial t} = \Delta_{\mu} f_t 
    \]
    with initial condition $f(t=0)= f_0$, 
    then there exists constant $C_5\textgreater0$ such that Dirichlet energy of $f_t$ will decay exponentially when $t$ tends to infinity
    \[
        \int_{G}\Vert\nabla_{\mu} f_t\Vert_2^2d\mu\leq e^{-2\lambda_1(-\Delta_{\mu}) t}\int_{G}\Vert\nabla_{\mu} f_0\Vert_2^2d\mu
    \]
    Suppose further $\forall i\in[n],\sum_{j\in\mathcal{N}_i}\omega_{ij}\leq\mu_{i}$, Let random walk matrix be $P_{\mu}=\Delta_{\mu}+I_n$, then for the arbitrary function $f$ on the graph, we have the decay rate as follows
    \[
   \int_{G}\Vert \nabla_{\mu} P^kf\Vert_2^2d\mu\leq(1-(2-\lambda_N(-\Delta_{\mu}))\lambda_1(-\Delta_{\mu}))^k\int_{G}\Vert \nabla_{\mu} f\Vert_2^2d\mu
    \]
    where $k\in\mathbb{N}^{*}$, specifically assuming the graph is bipartie, then $\lambda_{N}\neq 2$. Also for $\tilde{A}_{sym}$ we specify $\omega_{ij}=1,\mu_i=D_i+1$, we have decay rate
    \[
   \int_{G}\Vert \nabla_{\mu}(\tilde{D}^{-\frac{1}{2}} \tilde{A}_{sym}^kf)\Vert_2^2d\mu\leq(1-(2-\lambda_N(-\tilde{\Delta}_{rw-adj}))\lambda_1(-\tilde{\Delta}_{rw-adj}))^k\int_{G}\Vert \nabla_{\mu} (\tilde{D}^{-\frac{1}{2}}f)\Vert_2^2d\mu
    \]
\end{theorem}
We remark that the decay rate established in Theorem 4.1 is sharp. 
For the continuous heat evolution, consider the solution $f_t = v_1 e^{-\lambda_1(-\Delta_{\mu}) t}$, where $v_1$ is the eigenfunction corresponding to the eigenvalue $\lambda_1(-\Delta_{\mu})$. It is evident that $f_t$ satisfies the heat equation and yields the equality:
\[
\int_{G}\Vert\nabla_{\mu} f_t\Vert_2^2d\mu = e^{-2\lambda_{1}(-\Delta_{\mu})t}\int_{G}\Vert\nabla_{\mu} f_0\Vert_2^2d\mu.
\]
For the discrete random walk, consider a bipartite graph $K_{2,2}$ with vertex set $V=\{1,2,3,4\}$ and edges $E=\{(1,3),(1,4),(2,3),(2,4)\}$. Given a feature signal $X$ such that $X(1)=X(2)=1$ and $X(3)=X(4)=0$, we observe that the energy remains invariant: $\int_{G}\Vert \nabla_{\mu} A_{rw}X\Vert_2^2d\mu=\int_{G}\Vert \nabla_{\mu} X\Vert_2^2d\mu$. Consequently, assuming the graph is non-bipartite (implying $\lambda_N \neq 2$), over-smoothing occurs as the energy strictly exponentially decays.

The spectral gap $\lambda_1(-\Delta_{\mu})$ is intrinsically linked to the trade-off between over-smoothing and over-squashing~\citep{karhadkar2023fosr}. Specifically, a larger $\lambda_1$ mitigates over-squashing but exacerbates over-smoothing. Theorem 4.1 theoretically corroborates this trade-off in the context of general reversible random walks.
\end{section}
\section{Experimental verification }
In the section, we will experiment that vanilla attention based graph neural networks will be over-smooth with respect to Laplacian energy. We conduct our model on three well-known datasets: Cora, Citeseer, and Pubmed \citep{yang2016revisiting}. We employ Graph Convolutional Networks (GCN) and Graph Attention Networks (GAT) as our backbone architectures. As illustrated in Figure 1, when the network depth reaches 256 layers, the Laplacian energy drops below $10^{-14}$, indicating the onset of over-smoothing as measured by Laplacian energy.
\begin{figure}
  \centering
  \includegraphics[width=1\textwidth]{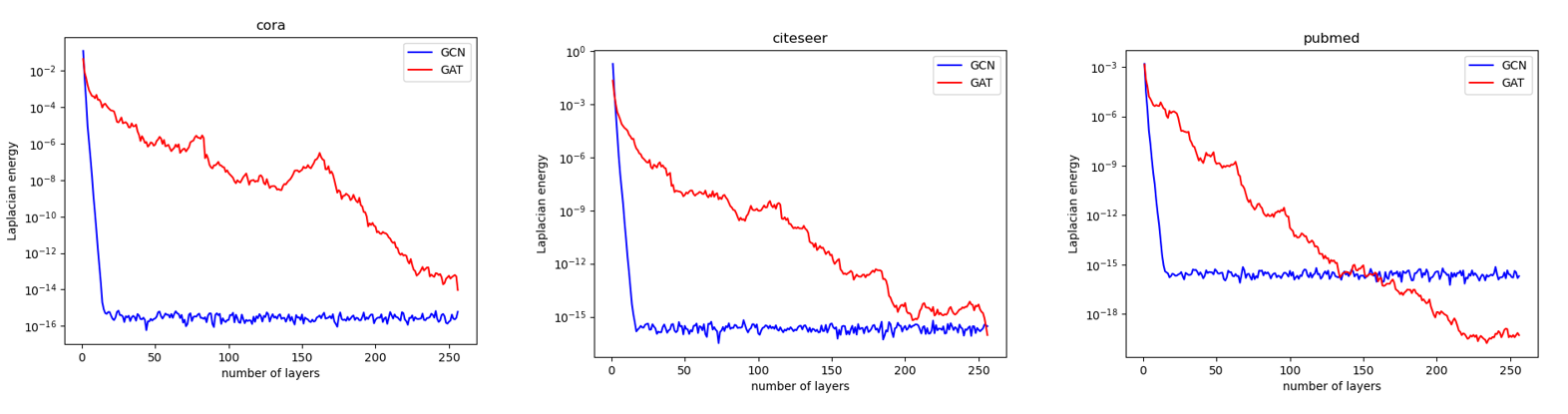}
  \caption{Evolution of Laplacian energy on the Cora dataset at initialization (y-axis in log scale).}
\end{figure}

\section{Conclusion}
We have proposed a series of node similarity measures. We believe that they will be useful and inspire us to address over-smoothing. More works need to be done to explore them.

\bibliographystyle{apalike}

\medskip

\begin{thebibliography}{}

\bibitem[Behmanesh et~al., 2023]{behmanesh2023tide}
Behmanesh, M., Krahn, M., and Ovsjanikov, M. (2023).
\newblock Tide: Time derivative diffusion for deep learning on graphs.

\bibitem[Cai and Wang, 2020]{cai2020note}
Cai, C. and Wang, Y. (2020).
\newblock A note on over-smoothing for graph neural networks.

\bibitem[Chamberlain et~al., 2021]{chamberlain2021grand}
Chamberlain, B.~P., Rowbottom, J., Gorinova, M.~I., Webb, S.~D., Rossi, E., and Bronstein, M.~M. (2021).
\newblock {GRAND}: Graph neural diffusion.
\newblock In {\em The Symbiosis of Deep Learning and Differential Equations}.

\bibitem[Chen et~al., 2019]{Chen2019MeasuringAR}
Chen, D., Lin, Y., Li, W., Li, P., Zhou, J., and Sun, X. (2019).
\newblock Measuring and relieving the over-smoothing problem for graph neural networks from the topological view.
\newblock In {\em AAAI Conference on Artificial Intelligence}.

\bibitem[Choi et~al., 2023]{pmlr-v202-choi23a}
Choi, J., Hong, S., Park, N., and Cho, S.-B. (2023).
\newblock {GREAD}: Graph neural reaction-diffusion networks.
\newblock In Krause, A., Brunskill, E., Cho, K., Engelhardt, B., Sabato, S., and Scarlett, J., editors, {\em Proceedings of the 40th International Conference on Machine Learning}, volume 202 of {\em Proceedings of Machine Learning Research}, pages 5722--5747. PMLR.

\bibitem[Fan et~al., 2022]{9139346}
Fan, W., Ma, Y., Li, Q., Wang, J., Cai, G., Tang, J., and Yin, D. (2022).
\newblock A graph neural network framework for social recommendations.
\newblock {\em IEEE Transactions on Knowledge and Data Engineering}, 34(5):2033--2047.

\bibitem[Giraldo et~al., 2023]{10.1145/3583780.3614997}
Giraldo, J.~H., Skianis, K., Bouwmans, T., and Malliaros, F.~D. (2023).
\newblock On the trade-off between over-smoothing and over-squashing in deep graph neural networks.
\newblock In {\em Proceedings of the 32nd ACM International Conference on Information and Knowledge Management}, CIKM '23, page 566–576, New York, NY, USA. Association for Computing Machinery.

\bibitem[Gligorijevi{\'c} et~al., 2021]{gligorijevic2021structure}
Gligorijevi{\'c}, V., Renfrew, P.~D., Kosciolek, T., Leman, J.~K., Berenberg, D., Vatanen, T., Chandler, C., Taylor, B.~C., Fisk, I.~M., Vlamakis, H., et~al. (2021).
\newblock Structure-based protein function prediction using graph convolutional networks.
\newblock {\em Nature communications}, 12(1):3168.

\bibitem[Grigor'yan et~al., 2016]{GRIGORYAN20164924}
Grigor'yan, A., Lin, Y., and Yang, Y. (2016).
\newblock Yamabe type equations on graphs.
\newblock {\em Journal of Differential Equations}, 261(9):4924--4943.

\bibitem[Horn et~al., 2019]{HornLinLiuYau+2019+89+130}
Horn, P., Lin, Y., Liu, S., and Yau, S.-T. (2019).
\newblock Volume doubling, poincaré inequality and gaussian heat kernel estimate for non-negatively curved graphs.
\newblock {\em Journal für die reine und angewandte Mathematik (Crelles Journal)}, 2019(757):89--130.

\bibitem[Jamadandi et~al., 2024]{jamadandi2024spectral}
Jamadandi, A., Rubio-Madrigal, C., and Burkholz, R. (2024).
\newblock Spectral graph pruning against over-squashing and over-smoothing.

\bibitem[Kang et~al., 2024]{kang2024unleashing}
Kang, Q., Zhao, K., Ding, Q., Ji, F., Li, X., Liang, W., Song, Y., and Tay, W.~P. (2024).
\newblock Unleashing the potential of fractional calculus in graph neural networks with {FROND}.
\newblock In {\em The Twelfth International Conference on Learning Representations}.

\bibitem[Karhadkar et~al., 2023]{karhadkar2023fosr}
Karhadkar, K., Banerjee, P.~K., and Montufar, G. (2023).
\newblock Fo{SR}: First-order spectral rewiring for addressing oversquashing in {GNN}s.
\newblock In {\em The Eleventh International Conference on Learning Representations}.

\bibitem[Kipf and Welling, 2017]{kipf2017semisupervised}
Kipf, T.~N. and Welling, M. (2017).
\newblock Semi-supervised classification with graph convolutional networks.
\newblock In {\em International Conference on Learning Representations}.

\bibitem[Li et~al., 2022]{pmlr-v162-li22h}
Li, M., Guo, X., Wang, Y., Wang, Y., and Lin, Z. (2022).
\newblock {G}$^2${CN}: Graph {G}aussian convolution networks with concentrated graph filters.
\newblock In Chaudhuri, K., Jegelka, S., Song, L., Szepesvari, C., Niu, G., and Sabato, S., editors, {\em Proceedings of the 39th International Conference on Machine Learning}, volume 162 of {\em Proceedings of Machine Learning Research}, pages 12782--12796. PMLR.

\bibitem[Maskey et~al., 2023]{NEURIPS2023_2a514213}
Maskey, S., Paolino, R., Bacho, A., and Kutyniok, G. (2023).
\newblock A fractional graph laplacian approach to oversmoothing.
\newblock In Oh, A., Naumann, T., Globerson, A., Saenko, K., Hardt, M., and Levine, S., editors, {\em Advances in Neural Information Processing Systems}, volume~36, pages 13022--13063. Curran Associates, Inc.

\bibitem[Nguyen et~al., 2023]{10.5555/3618408.3619488}
Nguyen, K., Nong, H., Nguyen, V., Ho, N., Osher, S., and Nguyen, T. (2023).
\newblock Revisiting over-smoothing and over-squashing using ollivier-ricci curvature.
\newblock In {\em Proceedings of the 40th International Conference on Machine Learning}, ICML'23. JMLR.org.

\bibitem[Oono and Suzuki, 2020]{Oono2020Graph}
Oono, K. and Suzuki, T. (2020).
\newblock Graph neural networks exponentially lose expressive power for node classification.
\newblock In {\em International Conference on Learning Representations}.

\bibitem[Rusch et~al., 2023a]{rusch2023survey}
Rusch, T.~K., Bronstein, M.~M., and Mishra, S. (2023a).
\newblock A survey on oversmoothing in graph neural networks.

\bibitem[Rusch et~al., 2022]{pmlr-v162-rusch22a}
Rusch, T.~K., Chamberlain, B., Rowbottom, J., Mishra, S., and Bronstein, M. (2022).
\newblock Graph-coupled oscillator networks.
\newblock In Chaudhuri, K., Jegelka, S., Song, L., Szepesvari, C., Niu, G., and Sabato, S., editors, {\em Proceedings of the 39th International Conference on Machine Learning}, volume 162 of {\em Proceedings of Machine Learning Research}, pages 18888--18909. PMLR.

\bibitem[Rusch et~al., 2023b]{rusch2023gradient}
Rusch, T.~K., Chamberlain, B.~P., Mahoney, M.~W., Bronstein, M.~M., and Mishra, S. (2023b).
\newblock Gradient gating for deep multi-rate learning on graphs.
\newblock In {\em The Eleventh International Conference on Learning Representations}.

\bibitem[Thorpe et~al., 2022]{thorpe2022grand}
Thorpe, M., Nguyen, T.~M., Xia, H., Strohmer, T., Bertozzi, A., Osher, S., and Wang, B. (2022).
\newblock {GRAND}++: Graph neural diffusion with a source term.
\newblock In {\em International Conference on Learning Representations}.

\bibitem[Veličković et~al., 2018]{veličković2018graph}
Veličković, P., Cucurull, G., Casanova, A., Romero, A., Liò, P., and Bengio, Y. (2018).
\newblock Graph attention networks.
\newblock In {\em International Conference on Learning Representations}.

\bibitem[Wang et~al., 2021]{wang2021dissecting}
Wang, Y., Wang, Y., Yang, J., and Lin, Z. (2021).
\newblock Dissecting the diffusion process in linear graph convolutional networks.
\newblock In Beygelzimer, A., Dauphin, Y., Liang, P., and Vaughan, J.~W., editors, {\em Advances in Neural Information Processing Systems}.

\bibitem[Wang et~al., 2023]{wang2023acmp}
Wang, Y., Yi, K., Liu, X., Wang, Y.~G., and Jin, S. (2023).
\newblock {ACMP}: Allen-cahn message passing with attractive and repulsive forces for graph neural networks.
\newblock In {\em The Eleventh International Conference on Learning Representations}.

\bibitem[Wu et~al., 2019]{pmlr-v97-wu19e}
Wu, F., Souza, A., Zhang, T., Fifty, C., Yu, T., and Weinberger, K. (2019).
\newblock Simplifying graph convolutional networks.
\newblock In Chaudhuri, K. and Salakhutdinov, R., editors, {\em Proceedings of the 36th International Conference on Machine Learning}, volume~97 of {\em Proceedings of Machine Learning Research}, pages 6861--6871. PMLR.

\bibitem[Wu et~al., 2023]{NEURIPS2023_6e4cdfdd}
Wu, X., Ajorlou, A., Wu, Z., and Jadbabaie, A. (2023).
\newblock Demystifying oversmoothing in attention-based graph neural networks.
\newblock In Oh, A., Naumann, T., Globerson, A., Saenko, K., Hardt, M., and Levine, S., editors, {\em Advances in Neural Information Processing Systems}, volume~36, pages 35084--35106. Curran Associates, Inc.

\bibitem[Xhonneux et~al., 2020]{pmlr-v119-xhonneux20a}
Xhonneux, L.-P., Qu, M., and Tang, J. (2020).
\newblock Continuous graph neural networks.
\newblock In III, H.~D. and Singh, A., editors, {\em Proceedings of the 37th International Conference on Machine Learning}, volume 119 of {\em Proceedings of Machine Learning Research}, pages 10432--10441. PMLR.

\bibitem[Yang et~al., 2016]{yang2016revisiting}
Yang, Z., Cohen, W.~W., and Salakhutdinov, R. (2016).
\newblock Revisiting semi-supervised learning with graph embeddings.

\bibitem[Zhao and Akoglu, 2020]{Zhao2020PairNorm:}
Zhao, L. and Akoglu, L. (2020).
\newblock Pairnorm: Tackling oversmoothing in gnns.
\newblock In {\em International Conference on Learning Representations}.

\bibitem[Zhou et~al., 2020]{10.5555/3495724.3496137}
Zhou, K., Huang, X., Li, Y., Zha, D., Chen, R., and Hu, X. (2020).
\newblock Towards deeper graph neural networks with differentiable group normalization.
\newblock In {\em Proceedings of the 34th International Conference on Neural Information Processing Systems}, NIPS '20, Red Hook, NY, USA. Curran Associates Inc.

\bibitem[Zhou et~al., 2021]{NEURIPS2021_b6417f11}
Zhou, K., Huang, X., Zha, D., Chen, R., Li, L., Choi, S.-H., and Hu, X. (2021).
\newblock Dirichlet energy constrained learning for deep graph neural networks.
\newblock In Ranzato, M., Beygelzimer, A., Dauphin, Y., Liang, P., and Vaughan, J.~W., editors, {\em Advances in Neural Information Processing Systems}, volume~34, pages 21834--21846. Curran Associates, Inc.

\end{thebibliography}

\small


\appendix

\section{Experiment details}

\subsection{Running Environment}
All our normalization methods and base models are implemented in PyTorch. Experiment that apply GCN with bias with different normalization methods on dataset Cora with no missing features and experiment that apply GAT with bias with different normalization methods on dataset Citeseer with 100$\%$ missing features are conducted on a machine with NVIDIA GeForce RTX 3070 Ti Laptop 16 GB GPU and 12th Gen Intel(R) Core(TM) i7-12700H CPU. Other experiments are conducted on a machine with NVIDIA GeForce RTX 3080 Laptop 16 GB GPU and AMD Ryzen 9 5900HX with Radeon Graphics CPU.
\subsection{Dataset Statistics}
\begin{table}[h!]
  \caption{Dataset Statistics}
  \label{sample-table}
  \centering
  \begin{tabular}{llll}
    \toprule
    Datasets &Cora &Citeseer &Pubmed\\
    \midrule
    Nodes &2708&3327&19717\\
    Edges &5429&4732&44338\\
    Features &1433&3703&500\\
    Classes  &7&6&3\\
    \bottomrule
  \end{tabular}
\end{table}

\section{Proofs of theorems}
We denote the bounds of weights as $\omega_{\max} = \max_{i,j} \omega_{ij}$, $\omega_{\min} = \min_{i,j} \omega_{ij}$, $\mu_{\max} = \max_{i} \mu_{i}$, and $\mu_{\min} = \min_{i} \mu_{i}$. 
\subsection{Proof of Theorem 2.1}
    Given a weighted graph $G=(V,E,\omega,\mu)$ and vectored-valued functions $f,g: V\rightarrow\mathbb{R}^{d}$. First, we suppose $d=1$, we have the following equality
    \begin{align*}
  &\sum_{i\in[n]}\sum_{j\in\mathcal{N}_i}\omega_{ij}f(i)g(i)\\
        &=\sum_{i\sim j}\omega_{ij}(f(i)g(i)+f(j)g(j))\\
        &=\sum_{i\in[n]}\sum_{j\in\mathcal{N}_i}\omega_{ij}f(j)g(j)
    \end{align*}
    and 
    \begin{align*}
    &\sum_{i\in[n]}\sum_{j\in\mathcal{N}_i}\omega_{ij}f(j)g(i)\\
        &=\sum_{i\sim j}\omega_{ij}(f(j)g(i)+f(i)g(j))\\
        &=\sum_{i\in[n]}\sum_{j\in\mathcal{N}_i}\omega_{ij}f(i)g(j)
    \end{align*}
    Combining above equalities we deduce
    \begin{align*}
      \int_{G}\Delta_{\mu} f\cdot gd\mu 
      &= \sum_{i=1}^n\sum_{j\in\mathcal{N}_i}\omega_{ij}(f(j)-f(i))g(i)\\
         & = \sum_{i=1}^n\sum_{j\in\mathcal{N}_i}(\omega_{ij}f(j)g(i)-\omega_{ij}f(i)g(i))\\
         & = \sum_{i=1}^n\sum_{j\in\mathcal{N}_i}(\omega_{ij}f(i)g(j)-\omega_{ij}f(j)g(j))\\
         &=\sum_{i=1}^n\sum_{j\in\mathcal{N}_i}\omega_{ij}(f(i)-f(j))g(j)\\       
    \end{align*}

    Thus adding the right-hand side of the first line and last line we get
    \begin{equation*}
        \begin{aligned}
        \int_{G}\Delta_{\mu} f\cdot gd\mu 
        &= \sum_{i=1}^n\sum_{j\in\mathcal{N}_i}\frac{\omega_{ij}}{2}(f(j)-f(i))(g(i)-g(j))\\&=-\int_{G}\nabla_{\mu} f\cdot\nabla_{\mu} g
    \end{aligned}
    \end{equation*}

    Now suppose general $d$ and $f=(f_1,\dots,f_d)^T, g=(g_1,\dots,g_d)^T$, then a direct calculation yield 
    \begin{align*}
        &\int_V\Delta_{\mu} f\cdot gd\mu\\
        &=\int_V\Delta_{\mu} f_1\cdot g_1d\mu+\dots+\Delta_{\mu} f_d\cdot g_dd\mu\\
        &=-\int_V\nabla_{\mu} f_1\cdot\nabla_{\mu} g_1d\mu-\dots-\nabla_{\mu} f_d\cdot \nabla_{\mu}g_dd\mu\\
        &=-\int_V\nabla_{\mu} f\cdot\nabla_{\mu} gd\mu
    \end{align*}
\subsection{Proof of Theorem 2.2}
Suppose $\alpha_1\leq\dots\leq\alpha_{nd}$ are eigenvalues with multiplicity and $v_1,\dots,v_{nd}$ are eigenvectors corresponding to eigenvalues respectively such that $\int_{G}\Vert v_k\Vert_2^2d\mu=1,\enspace\forall 1\leq k\leq nd$ and $\int_{G}v_i\cdot v_jd\mu=0,\enspace \forall i\neq j$. Apparently $\lambda_N=\alpha_{nd}$. For any vector-valued function $f: V\rightarrow\mathbb{R}^{d}$, suppose $f=c_1v_1+\dots+c_{nd}v_{nd}$, then we have  
\begin{align*}
    \int_{G}\Vert \nabla_{\mu}f\Vert_2^2d\mu&=\int_{G}-\Delta_{\mu}f\cdot fd\mu\\
    &=\alpha_1c_1^2+\cdots+\alpha_{nd}c_{nd}^2\\
    &\leq \alpha_{nd}(c_1^2+\cdots+c_{nd}^2)\\
    &=\lambda_N\int_{G}\Vert f\Vert_2^2
\end{align*}
Obviously we have $\int_{G}\Vert \nabla_{\mu}v_{nd}\Vert_2^2d\mu=\lambda_N\int_{G}\Vert v_{nd}\Vert_2^2$. We compute that for all vector-valued function $g: V\rightarrow\mathbb{R}^{d}$
\begin{align*}
    \int_{G}\Vert \nabla_{\mu}g\Vert_2^2d\mu&=\sum_{i\in[n]}(\sum_{j\in\mathcal{N}_i}\frac{1}{2}\omega_{ij}\Vert g(j)-g(i)\Vert_2^2)\\
    &\leq\sum_{i\in[n]}(\sum_{j\in\mathcal{N}_i}\omega_{ij}(\Vert g(j)\Vert_2^2+\Vert g(i)\Vert_2^2))\\
    &=2\sum_{i\in[n]}(\frac{\sum_{j\in\mathcal{N}_i}\omega_{ij}}{\mu_i})\Vert g(i)\Vert_2^2\mu_i\\
    &\leq 2M_{max}\int_{G}\Vert g\Vert_2^2d\mu
\end{align*}
Therefore we have $\lambda_N\leq 2M_{max}$
\subsection{Proof of Theorem 3.1}
(1)$\enspace$Suppose $f: V\rightarrow\mathbb{R}^d$ is a constant vector-valued function on the connected graph $G=(V, E,\omega,\mu)$, we wish to prove that $\mathcal{E}_{m}(f)=0$. Case $m\leq 1$ is well known. For case $m\geq 2$, we compute that
\[
\Delta_{\mu}X(i)=\frac{\omega_{ij}\big{(}X(j)-X(i)\big{)}}{\mu_{i}}=0
\]
Therefore if $m$ is even, $\mathcal{E}_m(X)=\int_{G}\Vert \Delta_{\mu}^{\frac{m}{2}}X\Vert_2^2d\mu=0$. If $m$ is odd, $\mathcal{E}_m(X)=\int_{G}\Vert \nabla_{\mu}\Delta_{\mu}^{\frac{m-1}{2}} X\Vert_2^2d\mu=0$. Now Suppose that for some $m$, $\mathcal{E}_m(f)=0$. we prove inductively that $f$ is a constant function on the graph. First, suppose $d=1$ , case $m=1$ is obvious. If $m=2$, then $\Delta_{\mu} X=0$. Suppose $f$ attains its maximum at node i, then we have \[
0=\Delta_{\mu} f(i)=\sum_{j\in\mathcal{N}_i}\frac{\omega_{ij}}{\mu_i}(f(j)-f(i))\leq 0
\]
From this we know that for all $j\in\mathcal{N}_i$, $f(j) = f(i)$. Because $G$ is connected, we conclude that for all $k\in V$, $f(k)=X(i)$. thus $f$ is a constant function. If $m=3$,  then $\Delta_{\mu} f$ is a constant function, suppose $\Delta_{\mu} f=c$. If $c > 0$, then similarly we suppose f attains its maximum at node i, then 
\[c=\Delta_{\mu} f(i)=\sum_{j\in\mathcal{N}_i}\frac{\omega_{ij}}{\mu_i}(f(j)-f(i))\leq 0
\]
this is a contradiction. If $c< 0$, then similarly we suppose $f$ attains its minimum at node i, then 
\[c=\Delta_{\mu} f(i)=\sum_{j\in\mathcal{N}_i}\frac{\omega_{ij}}{\mu_i}(f(j)-f(i))\leq 0
\]
this is also a contradiction. Thus $c=0$, from the analysis when $m=2$ we know that $f$ is a constant function. For general $m$, if $m$ is even, then $\Delta_{\mu}^{\frac{m}{2}}f=0$, using results when $m=2,3$ inductively we conclude $f$ is a constant function. If $m$ is odd, then $\Delta_{\mu}^{\frac{m-1}{2}}f$ is a constant vector, using results when $m=2,3$, we conclude that $f$ is a constant.

Now suppose for general dimension $d$, $f=(f_1,\dots,f_d)^T$. Then for arbitrary $m$, if $m$ is even. $\Delta_{\mu}^{\frac{m}{2}}f=(\Delta_{\mu}^{\frac{m}{2}}f_1,\dots,\Delta_{\mu}^{\frac{m}{2}}f_d)^T=0$. Therefore for all $k$, $\Delta_{\mu}^{\frac{m}{2}}f_k=0$, thus $f_k$ is a constant, so $f$ is a constant function. If $m$ is odd, the proof is similar.

(2)$\enspace$ For arbitrary $m$ and function $f^1,f^2$, we wish to prove 
\[
\gamma_{m}(f^1+f^2)\leq\gamma_{m}(f^1)+\gamma_{m}(f^2)
\]
For $m=0$, we compute by Minkowski inequality
\begin{align*}
    \gamma_{0}(f^1+f^2)&=\bigg{(}\int_{G}\left\Vert (f^1+f^2)-\frac{1}{\vert V\vert_{\mu}}\int_{G}(f^1+f^2)d\mu\right\Vert_2^2d\mu\bigg{)}^{\frac{1}{2}}\\
    &=\bigg{(}\sum_{i\in[n]}\sum_{1\leq k\leq d}\left\vert f_k^1(i)-\frac{1}{\vert V\vert_{\mu}}\int_{G}f_k^1d\mu+f^2_k(i)-\frac{1}{\vert V\vert_{\mu}}\int_{G}f_k^2d\mu\right\vert^2\mu_i\bigg{)}^{\frac{1}{2}}\\
    &\leq\bigg{(}\sum_{i\in[n]}\sum_{1\leq k\leq d}\big{(}\left\vert f_k^1(i)-\frac{1}{\vert V\vert_{\mu}}\int_{G}f_k^1d\mu\right\vert\sqrt{\mu_i}+\left\vert f^2_k(i)-\frac{1}{\vert V\vert_{\mu}}\int_{G}f_k^2d\mu\right\vert\sqrt{\mu_i}\big{)}^2\bigg{)}^{\frac{1}{2}}\\
    &\leq\bigg{(}\sum_{i\in[n]}\sum_{1\leq k\leq d}(\vert f_k^1-\frac{1}{\vert V\vert_{\mu}}\int_{G}f_k^1d\mu\vert)^2\mu_i\bigg{)}^{\frac{1}{2}}+\bigg{(}\sum_{i\in[n]}\sum_{1\leq k\leq d}(\vert f^2_k-\frac{1}{\vert V\vert_{\mu}}\int_{G}f_k^2d\mu\vert)^2\mu_i\bigg{)}^{\frac{1}{2}}\\
    &=\gamma_{0}(f^1)+\gamma_{0}(f^2)
\end{align*}
For $m=1$, we compute by Minkowski inequality
\begin{align*}
    \gamma_{1}(f^1+f^2)&=\big{(}\int_{G}\Vert\nabla_{\mu}(f^1+f^2)\Vert_2^2d\mu\big{)}^{\frac{1}{2}}\\
    &=\bigg{(}\sum_{i\in[n]}\sum_{j\in\mathcal{N}_i}\sum_{1\leq k\leq d}\omega_{ij}\vert f^1_k(j)-f^1_k(i)+f^2_k(j)-f^2_k(i)\vert^2\bigg{)}^{\frac{1}{2}}\\
    &\leq\bigg{(}\sum_{i\in[n]}\sum_{j\in\mathcal{N}_i}\sum_{1\leq k\leq d}( \omega_{ij}^{\frac{1}{2}}\vert f^1_k(j)-f^1_k(i)\vert+\omega_{ij}^{\frac{1}{2}}\vert f^2_k(j)-f^2_k(i)\vert)^2\bigg{)}^{\frac{1}{2}}\\
    &\leq\bigg{(}\sum_{i\in[n]}\sum_{j\in\mathcal{N}_i}\sum_{1\leq k\leq d}( \omega_{ij}^{\frac{1}{2}}\vert f^1_k(j)-f^1_k(i)\vert)^2\bigg{)}^{\frac{1}{2}}+\bigg{(}\sum_{i\in[n]}\sum_{j\in\mathcal{N}_i}\sum_{1\leq k\leq d}( \omega_{ij}^{\frac{1}{2}}\vert f^2_k(j)-f^2_k(i)\vert)^2\bigg{)}^{\frac{1}{2}}\\
    &=\gamma_{1}(f^1)+\gamma_{1}(f^2)
\end{align*}
For $m=2$, we compute by Minkowski inequality
\begin{align*}
    \gamma_{2}(f^1+f^2)&=\big{(}\int_{G}\Vert\Delta_{\mu}(X^1+X^2)\Vert_2^2d\mu\big{)}^{\frac{1}{2}}\\
    &=\big{(}\sum_{i\in[n]}\Vert\Delta_{\mu}f^1(i)+\Delta_{\mu}f^2(i)\Vert_2^2\mu_i\big{)}^{\frac{1}{2}}\\
    &=\bigg{(}\sum_{i\in[n]}\sum_{1\leq k\leq d}\vert\Delta_{\mu}f_k^1(i)+\Delta_{\mu}f_k^2(i)\vert^2\mu_i\bigg{)}^{\frac{1}{2}}\\
    &=\bigg{(}\sum_{i\in[n]}\sum_{1\leq k\leq d}\vert\mu_i^{\frac{1}{2}}\Delta_{\mu}f_k^1(i)+\mu_i^{\frac{1}{2}}\Delta_{\mu}f_k^2(i)\vert^2\bigg{)}^{\frac{1}{2}}\\
    &\leq\bigg{(}\sum_{i\in[n]}\sum_{1\leq k\leq d}\vert\mu_i^{\frac{1}{2}}\Delta_{\mu}f_k^1(i)\vert^2\bigg{)}^{\frac{1}{2}}+\bigg{(}\sum_{i\in[n]}\sum_{1\leq k\leq d}\vert\mu_i^{\frac{1}{2}}\Delta_{\mu}f_k^2(i)\vert^2\bigg{)}^{\frac{1}{2}}\\
    &=\gamma_{2}^2(X^1)+\gamma_{2}^2(X^2)
\end{align*}
For general $m\geq3$, if $m$ is even, then
\begin{align*}
    \gamma_{m}(f^1+f^2)&=(\int_{G}\Vert\Delta_{\mu}^{\frac{m}{2}}(f^1+f^2)\Vert_2^2d\mu)^{\frac{1}{2}}\\
    &=\big{(}\int_{G}\Vert\Delta_{\mu}(\Delta_{\mu}^{\frac{m}{2}-1}f^1+\Delta_{\mu}^{\frac{m}{2}-1}f^2)\Vert_2^2d\mu\big{)}^{\frac{1}{2}}\\
    &\leq\big{(}\int_{G}\Vert\Delta_{\mu}(\Delta_{\mu}^{\frac{m}{2}-1}f^1)\Vert_2^2d\mu\big{)}^{\frac{1}{2}}+\big{(}\int_{G}\Vert\Delta_{\mu}(\Delta_{\mu}^{\frac{m}{2}-1}f^2)\Vert_2^2d\mu\big{)}^{\frac{1}{2}}\\
    &=\gamma_{m}(f^1)+\gamma_{m}(f^2)
\end{align*}
If $m$ is odd, then
\begin{align*}
    \gamma_{m}(f^1+f^2)&=\big{(}\int_{G}\Vert\nabla_{\mu}\Delta_{\mu}^{\frac{m-1}{2}}(f^1+f^2)\Vert_2^2d\mu\big{)}^{\frac{1}{2}}\\
    &=\big{(}\int_{G}\Vert\nabla_{\mu}(\Delta_{\mu}^{\frac{m-1}{2}}f^1+\Delta_{\mu}^{\frac{m-1}{2}}f^2)\Vert_2^2d\mu\big{)}^{\frac{1}{2}}\\
    &\leq\big{(}\int_{G}\Vert\nabla_{\mu}(\Delta_{\mu}^{\frac{m-1}{2}}f^1)\Vert_2^2d\mu\big{)}^{\frac{1}{2}}+\big{(}\int_{G}\Vert\nabla_{\mu}(\Delta_{\mu}^{\frac{m-1}{2}}f^2)\Vert_2^2d\mu\big{)}^{\frac{1}{2}}\\
    &=\gamma_{m}(f^1)+\gamma_{m}(f^2)
\end{align*}
\subsection{Proof of Theorem 3.2}
Given a weighted connected graph $G=(V, E,\omega,\mu)$ and a vector-valued function $f: V\rightarrow\mathbb{R}^d$. Suppose $0=\alpha_1\leq\alpha_2\leq\dots\leq\alpha_{nd}$ are eigenvalues with multiplicity and $v_1,\dots,v_{nd}$ are eigenvectors corresponding to eigenvalues respectively such that $\int_{G}\Vert v_k\Vert_2^2d\mu=1,\enspace\forall 1\leq k\leq nd$ and $\int_{G}v_i\cdot v_jd\mu=0,\enspace \forall i\neq j$. Since the graph is connected, the multiplicity of $0$ is $d$. Let $\lambda_1$ be the first nonzero eigenvalue, then we have
\begin{align*}
    \int_{G}\Vert \nabla_{\mu}f\Vert_2^2d\mu&=-\int_{G}\Delta_{\mu}f\cdot fd\mu\\
    &=\sum_{d+1\leq k\leq nd}\alpha_kc_k^2\\
    &\geq \lambda_1\sum_{d+1\leq k\leq nd}c_k^2\\
    &=\lambda_1\int_{G}\Vert f-\frac{1}{\vert V\vert_{\mu}}\int_{G}fd\mu\Vert_2^2d\mu
\end{align*}

Next, we wish to prove  
\[
    \int_{G}\Vert\nabla_{\mu} f\Vert_2^2d\mu\leq C\int_{G}\Vert f-\frac{1}{\vert V\vert_{\mu}}\int_{G}fd\mu\Vert_2^2d\mu
\]
We directly compute that
\begin{align*}
    \int_{G}\Vert \nabla_{\mu}g\Vert_2^2d\mu&=\sum_{i\in[n]}\sum_{j\in\mathcal{N}_i}\frac{\omega_{ij}}{2}\Vert g(j)-g(i)\Vert_2^2\\
    &=\sum_{i\in[n]}\sum_{j\in\mathcal{N}_i}\frac{\omega_{ij}}{2}(\sum_{1\leq k\leq d}\vert g_k(j)-g_k(i)\vert^2)\\
    &\leq \sum_{i\in[n]}\sum_{j\in\mathcal{N}_i}\omega_{ij}(\sum_{1\leq k\leq d}(\vert g_k(j)\vert^2+\vert g_k(i)\vert^2))\\
    &=2\sum_{i\in[n]}(\frac{\sum_{j\in\mathcal{N}_i}\omega_{ij}}{\mu_i})\Vert g(i)\Vert_2^2\mu_i\\
    &\leq 2M_{max}\int_{G}\Vert g\Vert_2^2d\mu
\end{align*}
Replace $g$ by $f-\frac{1}{\vert V\vert}\int_{G}fd\mu$ we have the desired result
\subsection{Proof of Theorem 3.3}
Given a weighted graph $G=(V, E,\omega,\mu)$ and a vector-valued function $f: V\rightarrow\mathbb{R}^{d}$, First we wish to prove
\begin{equation}
     C_2\int\Vert\nabla_{\mu} f\Vert_2^2d\mu\leq\int\Vert\Delta_{\mu} f\Vert_2^2d\mu\leq C_1\int\Vert\nabla_{\mu} f\Vert_2^2d\mu
\end{equation}
First, we suppose $d=1$, for the right-hand side inequality, suppose $\mu_i^1=\sum\limits_{j\in\mathcal{N}_i}\omega_{ij}$ we directly compute that 
\begin{align*}
    \int_{G}\Vert\Delta_{\mu}f\Vert_2^2d\mu&=\sum_{i\in[n]}\bigg{(}\sum\limits_{j\in\mathcal{N}_i}\frac{\omega_{ij}\big{(}f(j)-f(i)\big{)}}{\mu_i}\bigg{)}^2\mu_i\\
    &=\sum_{i\in[n]}\bigg{(}\sum\limits_{j\in\mathcal{N}_i}\frac{\omega_{ij}\big{(}f(j)-f(i)\big{)}}{\mu^1_i}\bigg{)}^2\frac{(\mu^1_i)^2}{\mu_i}\\
    &\leq \sum_{i\in[n]}\sum\limits_{j\in\mathcal{N}_i}\omega_{ij}\big{(}f(j)-f(i)\big{)}^2\frac{\mu^1_i}{\mu_i}\\
    &\leq 2M_{max}\int_{G}\Vert\nabla_{\mu}f\Vert_2^2d\mu
\end{align*}
The third line is from Jenson inequality. For the left-hand side, we have
\begin{align*}
\int_{G}\Vert\nabla_{\mu}f\Vert_2^2d\mu&=\int_{G}\left\Vert\nabla_{\mu}(f-\frac{1}{\vert V\vert_{\mu}}\int_{G}fd\mu)\right\Vert_2^2d\mu\\
    &= - \int_{G}\Delta_{\mu}(f-\frac{1}{\vert V\vert_{\mu}}\int_{G}fd\mu)\cdot (f-\frac{1}{\vert V\vert_{\mu}}\int_{G}fd\mu)d\mu\\
    &\leq\sqrt{\int_{G}\Vert\Delta_{\mu}f\Vert_2^2d\mu}\sqrt{\int_{G}\left\Vert f-\frac{1}{\vert V\vert_{\mu}}\int_{G}fd\mu\right\Vert_2^2d\mu}\\
    &\leq \sqrt{\int_{G}\Vert\Delta_{\mu}f\Vert_2^2d\mu}\sqrt{\frac{1}{\lambda_1}\int_{G}\Vert\nabla_{\mu}f\Vert_2^2d\mu}
\end{align*}
The third Line is from Holder inequality. Therefore we have 
\[
\int_{G}\Vert\nabla_{\mu}f\Vert_2^2d\mu\leq \frac{1}{\lambda_1}\int_{G}\Vert\Delta_{\mu}f\Vert_2^2d\mu
\]
For general dimension $d$, suppose $f=(f_1,\dots,f_d)^T$, then by Jenson inequality we have
\begin{align*}
    \int_{G}\Vert\Delta_{\mu}f\Vert_2^2d\mu &= \sum_{i\in[n]}\bigg{(}\sum_{1\leq k \leq d}\sum_{j\in\mathcal{N}_i}\big{(}\frac{\omega_{ij}(f_k(j)-f_k(i))}{\mu_i}\big{)}^2\bigg{)}\mu_i\\
    &=\sum_{i\in[n]}\bigg{(}\sum_{1\leq k \leq d}\sum_{j\in\mathcal{N}_i}\big{(}\frac{\omega_{ij}(f_k(j)-f_k(i))}{\mu^1_i}\big{)}^2\bigg{)}\frac{(\mu^1_i)^2}{\mu_i}\\
    &\leq \sum_{i\in[n]}\bigg{(}\sum_{1\leq k \leq d}\sum_{j\in\mathcal{N}_i}\frac{\omega_{ij}\big{(}f_k(j)-f_k(i)\big{)}^2}{\mu^1_i}\bigg{)}\frac{(\mu^1_i)^2}{\mu_i}\\
    &=\mu_i^1\int_{G}\Vert\nabla_{\mu}f\Vert_2^2d\mu
\end{align*}
and 
\begin{align*}
    \int_{G}\Vert\Delta_{\mu}f\Vert_2^2d\mu &= \sum_{i\in[n]}\bigg{(}\sum_{1\leq k \leq d}\sum_{j\in\mathcal{N}_i}\big{(}\frac{\omega_{ij}(f_k(j)-f_k(i))}{\mu_i}\big{)}^2\bigg{)}\mu_i\\
    &\geq\lambda_1\sum_{1\leq k\leq d}\int_{G}\Vert\nabla_{\mu}f_k\Vert_2^2d\mu\\
    &=\lambda_1\int_{G}\Vert\nabla_{\mu}f\Vert_2^2d\mu
\end{align*}
Thus the proof of (1) is over. After we wish to prove 
\begin{equation}
    C_3\int\Vert\nabla_{\mu} \Delta_{\mu} f\Vert_2^2d\mu\leq\int\Vert\Delta_{\mu} f\Vert_2^2d\mu\leq C_4\int\Vert\nabla_{\mu} \Delta_{\mu} f\Vert_2^2d\mu
\end{equation}
        
For the right-hand side of (2), we first notice that
\[
    \int_{G}\Delta_{\mu}fd\mu=\int_{G}f\Delta_{\mu}1d\mu=0
\]
Hence by Poincare inequality, we have
\begin{align*}
    \int\Vert\Delta_{\mu} f\Vert_2^2d\mu&=\int\left\Vert\Delta_{\mu} f-\frac{1}{\vert V\vert_{\mu}}\int_{G}\Delta_{\mu}fd\mu\right\Vert_2^2d\mu\\
    &\leq \frac{1}{\lambda_1}\int\left\Vert\nabla_{\mu}\Delta_{\mu} f\right\Vert_2^2d\mu
\end{align*}
For the left-hand side of (2), we notice that for any vector-valued function $g$, we have
\begin{align*}
    \int\Vert\nabla_{\mu} g\Vert_2^2d\mu&=\sum_{i\in[n]}\sum_{j\in\mathcal{N}_i}\frac{\omega_{ij}}{2}\Vert g(j)-g(i)\Vert_2^2\\
    &\leq\sum_{i\in[n]}\sum_{j\in\mathcal{N}_i}\omega_{ij}(\Vert g(j)\Vert_2^2+\Vert g(i)\Vert_2^2)\\
    &=2\sum_{i\in[n]}(\frac{\sum_{j\in\mathcal{N}_i}\omega_{ij}}{\mu_i})\Vert g(i)\Vert_2^2\mu_i\\
    &\leq 2M_{max}\int_{G}\Vert g\Vert_2^2d\mu
\end{align*}
Replace $g$ by $\Delta_{\mu}f$ we have the left-hand side of (2). Now suppose $G=(V, E,\omega,\mu)$ is not connected and has connected components $V_1,\dots, V_l$, then for the right-hand side of (1) we have 
\begin{align*}
    \int_{G}\Vert\Delta_{\mu}f\Vert_2^2d\mu &= \sum_{1\leq k\leq l}\int_{V_k}\Vert\Delta_{\mu}f\Vert_2^2d\mu\\
    &\leq C_1\sum_{1\leq k\leq l}\int_{V_k}\Vert\nabla_{\mu}f\Vert_2^2d\mu\\
    &=C_1\int_{G}\Vert\nabla_{\mu}f\Vert_2^2d\mu
\end{align*}
Other cases are similar. In the end we show that $\mathcal{E}_W$ and $\mathcal{E}_D$ are special cases of $\mathcal{E}_m$. $\mathcal{E}_W$ is a special cases of $\mathcal{E}_m$ such that
\[
m=0, p=2, \mu_i=1\enspace \forall i\in[n]
\]
$\mathcal{E}_D$ is a special cases of $\mathcal{E}_m$ such that
\[
m=1, 2=2, \omega_{ij}=2, \mu_{i}=1\enspace \forall i\in[n],j\in\mathcal{N}_i
\]
.
\subsection{Proof of Theorem 3.4}
Suppose f has spectral decomposition 
\[
f=\sum_{1\leq k\leq nd}C_kv_k
\]
If m is even, then
\begin{align*}
    &\int_{G}\Vert\Delta^{\frac{m}{2}}_{\mu} f\Vert_2^2d\mu\\
    &=\int_{G}(\sum_{1\leq k\leq nd}\alpha_k^{\frac{m}{2}}C_kv_k)\cdot (\sum_{1\leq k\leq nd}\alpha_k^{\frac{m}{2}}C_kv_k)d\mu\\
    &=\sum_{1\leq k\leq nd}\alpha_k^{m}C_k^2
\end{align*}
If m is odd, then 
\begin{align*}
    &\int_{G}\Vert\nabla_{\mu}\Delta^{\frac{m-1}{2}}_{\mu} f\Vert_2^2d\mu\\
    &=\int_{G}-\Delta_{\mu}(\Delta_{\mu}^{\frac{m-1}{2}}f)\cdot \Delta_{\mu}^{\frac{m-1}{2}}fd\mu\\
    &=\int_{G}(\sum_{1\leq k\leq nd}\alpha_k^{\frac{m+1}{2}}C_kv_k)\cdot (\sum_{1\leq k\leq nd}\alpha_k^{\frac{m-1}{2}}C_kv_k)d\mu\\
    &=\sum_{1\leq k\leq nd}\alpha_k^{m}C_k^2
\end{align*}
\subsection{Proof of Theorem 4.1}
First we need a lemma.
\begin{lemma}
    $G=(V,E,\omega,\mu)$ is weighted connected graph such that $\sum\limits_{j\in\mathcal{N}_i}\omega_{ij}\leq \mu_i$, then the largest eigenvalue $\lambda_{N}=2$ if and only if $G$ is bipartite
\end{lemma}
\begin{proof}
    (1) If $\lambda_N=2$, let $v_N$ be the corresponding eigenfunction, then $ -\Delta_{\mu}v_N=2v_N$. Hence we have
    \[
    \int_{G}-\Delta_{\mu}v_N\cdot v_Nd\mu =2\int_{G}\Vert v_N\Vert_2^2d\mu
    \]
    Integrate by parts we have
    \[
    \int_{G}\Vert \nabla_{\mu} v_N\Vert_2^2d\mu=2\int_{G}\Vert v_N\Vert_2^2d\mu
    \]
    Hence, we have
    \[
    0=\sum_{i\sim j}\omega_{ij}(2\Vert v_N(i)\Vert^2+2\Vert v_N(j)\Vert_2^2-\Vert v_N(i)-v_N(j)\Vert^2_2)+2\sum_{i}(\mu_i-\sum_{j\in\mathcal{N}_i}\omega_{ij})\Vert v_N(i)\Vert_2^2
    \]
    But
    \[
    \begin{aligned}
        &\sum_{i\sim j}\omega_{ij}(2\Vert v_N(i)\Vert^2+2\Vert v_N(j)\Vert_2^2-\Vert v_N(i)-v_N(j)\Vert^2_2)+2\sum_{i}(\mu_i-\sum_{j\in\mathcal{N}_i}\omega_{ij})\Vert v_N(i)\Vert_2^2\\
        &=\sum_{i\sim j}\omega_{ij}\Vert v_N(i)+v_N(j)\Vert_2^2+2\sum_{i}(\mu_i-\sum_{j\in\mathcal{N}_i}\omega_{ij})\Vert v_N(i)\Vert_2^2\\
    \end{aligned}
    \]
    
    Hence, we have 
    \begin{equation}
    v_N(i)+v_N(j)=0,i\sim j     
    \end{equation}

    If there exist $i\in V$ such that $v_{N}(i)=0$, then by (5) we have $v_{N}(j)=0$ for all $j\in V$, which is a contradiction. Hence let $U=\{i\in V: v_N(i)>0\}$, $V=\{i\in V: v_{N}(i)<0\}$, we have the desired partition.
\end{proof}
Now we come back to the proof of Theorem 4.1
\begin{proof}
(1)$\enspace$Suppose $f$ is a solution to the heat equation with initial condition $f(t=0)=f_0$, then
\begin{align*}
    \frac{\partial \int_{G}\Vert \nabla_{\mu}f\Vert_2^2d\mu}{\partial t} &= 2\int_{G}\nabla_{\mu} f\cdot\nabla_{\mu}\partial_{t}fd\mu\\
    &=-2\int_{G}\Delta_{\mu} f\cdot\partial_{t}fd\mu\\
    &=-2\int_{G}\Vert\Delta_{\mu} f\Vert_2^2d\mu\\
    &\leq -2\lambda_1(-\Delta_{\mu})\int_{G}\Vert\nabla_{\mu} f\Vert_2^2d\mu
\end{align*}
Thus we have
\[
\frac{\partial(e^{2\lambda_1(-\Delta_{\mu})t}\int_{G}\Vert \nabla_{\mu}f\Vert_2^2d\mu)}{\partial t}\leq 0
\]
Therefore
\[
e^{2\lambda_1(-\Delta_{\mu})t}\int_{G}\Vert \nabla_{\mu}f\Vert_2^2d\mu\vert_{t=t}\leq \big{(}e^{2\lambda_1(-\Delta_{\mu})t}\int_{G}\Vert \nabla_{\mu}f\Vert_2^2d\mu\big{)}\vert_{t=0}
\]
In conclusion, we have
\[
\int_{G}\Vert \nabla_{\mu}f\Vert_2^2d\mu\leq e^{-2\lambda_1(-\Delta_{\mu})t}\int_{G}\Vert \nabla_{\mu}f_0\Vert_2^2d\mu
\]

(2)$\enspace$Suppose $f: V\rightarrow\mathbb{R}^{d}$ is an arbitrary vector-valued function on the graph. Then we directly compute that
\begin{align*}
\int_{G}\Vert\nabla_{\mu} Pf\Vert_2^2d\mu & = \int_{G}\nabla_{\mu} Pf\cdot \nabla_{\mu} Pfd\mu\\
& = -\int_{G}\Delta_{\mu} Pf\cdot  Pfd\mu\\
& = -\int_{G}\Delta_{\mu} (\Delta_{\mu}+I)f\cdot(\Delta_{\mu}+I)fd\mu\\
& =  -\int_{G}\Delta^2_{\mu} f\cdot\Delta_{\mu}fd\mu-\int_{G}\Delta^2_{\mu} f\cdot fd\mu-\int_{G}\Vert\Delta_{\mu}f\Vert_2^2d\mu-\int_{G}\Delta_{\mu}f\cdot fd\mu\\
& = \int_{G}\Vert\nabla_{\mu}\Delta_{\mu} f\Vert_{2}^2d\mu-2\int_{G}\Vert\Delta_{\mu}f\Vert_2^2d\mu+\int_{G}\Vert\nabla_{\mu}f\Vert_2^2d\mu\\
&\leq \lambda_N(-\Delta_{\mu})\int_{G}\Vert\Delta_{\mu} f\Vert_{2}^2d\mu-2\int_{G}\Vert\Delta_{\mu} f\Vert_{2}^2d\mu+\int_{G}\Vert\nabla_{\mu}f\Vert_2^2d\mu\\
&\leq(\lambda_N(-\Delta_{\mu})-2)\lambda_1(-\Delta_{\mu})\int_{G}\Vert\nabla_{\mu}f\Vert_2^2d\mu+\int_{G}\Vert\nabla_{\mu}f\Vert_2^2d\mu\\
&=(1+(\lambda_N(-\Delta_{\mu})-2)\lambda_1(-\Delta_{\mu}))\int_{G}\Vert\nabla_{\mu}f\Vert_2^2d\mu
\end{align*}
Therefore we have
\[
\int_{G}\Vert\nabla_{\mu} P^kf\Vert_2^2d\mu\leq\big{(}1+(\lambda_N(-\Delta_{\mu})-2)\lambda_1(-\Delta_{\mu})\big{)}^k\int_{G}\Vert\nabla_{\mu}f\Vert_2^2d\mu
\]
From Lemma B.1 we know that $\lambda_N=2$ if and only if $G$ is bipartite, hence if $G$ is non-bipartite we will have over-smoothing. For $\tilde{A}_{sym}$, we specify $P=\tilde{A}_{rw},\omega_{ij}=1,\mu_i=D_i+1$ and notice that $\tilde{A}_{sym}=\tilde{D}^{\frac{1}{2}}\tilde{A}_{rw}\tilde{D}^{-\frac{1}{2}}$. Therefore we have
\[
\int_{G}\Vert\nabla_{\mu} \tilde{D}^{-\frac{1}{2}}\tilde{A}_{sym}^k(\tilde{D}^{\frac{1}{2}}f)\Vert_2^2d\mu\leq(1+(\lambda_N(-\tilde{\Delta}_{rw-adj})-2)\lambda_1(-\tilde{\Delta}_{rw-adj}))^k\int_{G}\Vert\nabla_{\mu}f\Vert_2^2d\mu
\]
Replace $f$ by $\tilde{D}^{-\frac{1}{2}}g$ we get desired results.
\end{proof}

\end{document}